\documentclass{IOS-Book-Article}

\usepackage{mathptmx}
\usepackage{soul}\setuldepth{article}
\usepackage{graphicx}
\usepackage{amssymb}
\usepackage{mathrsfs}
\usepackage{amsmath} 
\usepackage{amsthm}
\newtheorem{definition}{Definition}[section]
\newtheorem{example}{Example}[section]
\newtheorem{principle}{Principle}[section]
\newtheorem{proposition}{Proposition}[section]
\usepackage[cal=cm]{mathalfa}
\newcommand{\goal}{\theta} 
\def\hb{\hbox to 11.5 cm{}}

\begin{document}

\pagestyle{headings}
\def\thepage{}
\begin{frontmatter}              

\title{Dispute resolution in legal mediation with quantitative argumentation}

\markboth{}{August 2024\hb}

\author{\fnms{Xiao} \snm{Chi}
\thanks{E-mail: cx3506@outlook.com .}}

\runningauthor{X. Chi}
\address{Zhejiang University}

\begin{abstract}
Mediation is often treated as an extension of negotiation, without taking into account the unique role that norms and facts play in legal mediation. Additionally, current approaches for updating argument acceptability in response to changing variables frequently require the introduction of new arguments or the removal of existing ones, which can be inefficient and cumbersome in decision-making processes within legal disputes. In this paper, our contribution is two-fold. First, we introduce a QuAM (Quantitative Argumentation Mediate) framework, which integrates the parties' knowledge and the mediator's knowledge, including facts and legal norms, when determining the acceptability of a mediation goal. Second, we develop a new formalism to model the relationship between the acceptability of a goal argument and the values assigned to a variable associated with the argument. 
We use a real-world legal mediation as a running example to illustrate our approach.
\end{abstract}

\end{frontmatter}
\markboth{August 2024\hb}{August 2024\hb}

\section{Introduction}

Dispute resolution plays an important role in modern society.  
Current research primarily concentrates on formal mechanisms of dispute resolution, including negotiation, mediation, arbitration, and litigation \cite{r23}. 
These mechanisms involve consensus issue of varying degrees \cite{r1}. 
In this work, we aim to model consensual mechanisms of dispute resolution in mediation. Unlike other mechanisms where a third party is involved,
consensus in mediation takes a central role, and a mediator does not hold  final decision-making power in resolving a dispute. 
Mediation typically falls into three categories: informal, administrative, and judicial. This work focuses on informal mediation, which deals with civil disputes. 
Normally, the mediation of civil disputes is conducted within the framework of civil laws, and adheres to the principle of autonomy of will. This means that, if the parties are willing, the law will not intervene proactively unless public order and good morals are violated. 

A core issue in legal mediation is combining concerns and knowledge of both parties with the facts and legal norms discovered by a mediator to effectively conduct mediation and achieve desired outcomes. 
The mediator has two primary functions: aiding the parties in generating options and helping to break through impasses \cite{r14}. While the former has been extensively studied \cite{r12,r13}, the latter is less explored. Breaking through impasses often involves a process where both parties gradually compromising on an issue. Hence, a mechanism that can quantify the goals of both parties step by step is required. 

Abstract argumentation provides an effective way for reasoning about conflict \cite{r2} and reaching consensus in disputes \cite{r3,r4,r5}. 
Among them, quantitative argumentation provides a method to dynamically quantify the strengths of alternative decision options by recursively aggregating the strengths of their supporters and attackers \cite{r16}. However, existing work has not yet applied quantitative argumentation in mediation. This paper aims to formalize a dispute resolution mechanism based on the quantitative argumentation debate (QuAD) framework to assist mediators throughout the mediation process. There are several issues to be considered when instantiating the QuAD framework in legal mediation. Firstly, in mediation, the degree of relevance between arguments may vary. Therefore, in addition to weighing arguments, we also need to assign weights to the relations between them. Secondly, in legal mediation, there are several types of arguments, including opinion arguments, factual arguments, mandatory arguments and dispositive arguments. The characteristics of these types of arguments need to be considered during the aggregation process. Lastly, we need to develop a formalism to model the relationship between the degree of argument acceptability and the values assigned to variables associated with arguments. This will enable an efficient and straightforward comparison of whether an issue has been resolved after mediation between the two parties.

This paper is organized as follows. Section 2 introduces some notions of the QuAD framework. Section 3 presents a quantitative argumentation mediate (QuAM) mechanism for assessing the conflict between disputing parties.
Section 4 introduces some properties of the QuAM mechanism. Section 5 concludes the paper and discusses related work and future work.


\section{The QuAD framework}

In this section, we introduce some existing notions of QuAD framework that will be used in the paper. For more details of QuAD framework, the readers are referred to  \cite{r16}.

A QuAD framework is a 5-tuple $\langle \mathcal{A},\mathcal{C},\mathcal{P},\mathcal{R},\mathcal{BS} \rangle$, where $\mathcal{A}$ is a finite set of \textit{answer arguments}, $\mathcal{C}$ is a finite set of \textit{con-arguments}, $\mathcal{P}$ is a finite set of \textit{pro-arguments}, the sets $\mathcal{A}$, $\mathcal{C}$, and $\mathcal{P}$ are pairwise disjoint; $\mathcal{R} \subseteq (\mathcal{C} \cup \mathcal{P}) \times (\mathcal{A} \cup \mathcal{C} \cup \mathcal{P})$ is an acyclic binary relation, $\mathcal{BS}:(\mathcal{A} \cup \mathcal{C} \cup \mathcal{P}) \rightarrow \mathbb{I}$ is a total function for scale $\mathbb{I} = [0,1]$; $\mathcal{BS}(a)$ is the \textit{base score} of argument $a$. 

Automatic evaluation in a QuAD framework uses a \textit{score function} $\mathcal{SF}$, which assigns a final score to answer nodes. $\mathcal{SF}$ is defined recursively, employing a \textit{score operator} that combines the base score of a node with the final scores of its attackers and supporters. In a dynamic design context, sets of attackers and supporters are actually be given in sequence, which are arbitrary permutations of the attackers and supporters.
Let $(a_{1},...,a_{n})$ (resp. $(b_{1},...,b_{m})$) $(n, m \geq 0)$ be an arbitrary permutation of the attackers (resp. supporters) in $\mathcal{R}^{-}(a)$ (resp. $\mathcal{R}^{+}(a)$), where $\mathcal{R}^{-}(a) = \{b \in \mathcal{C} \mid (b,a) \in \mathcal{R}\}$ is a set of direct attackers of $a$, and $\mathcal{R}^{+}(a)  = \{b \in \mathcal{P} \mid (b,a) \in \mathcal{R}\}$ is a set of direct supporters of $a$. The corresponding sequence of final scores for $\mathcal{R}^{-}(a)$ and $\mathcal{R}^{+}(a)$ are $SEQ_{\mathcal{SF}}(\mathcal{R}^{-}(a)) = (\mathcal{SF}(a_{1}),..., \mathcal{SF}(a_{n}))$ and $SEQ_{\mathcal{SF}}(\mathcal{R}^{+}(a)) = (\mathcal{SF}(b_{1}),..., \mathcal{SF}(b_{m}))$. A generic score function for an argument $a$ can be given as 
$$\mathcal{SF}(a) = g(\mathcal{BS}(a), \mathcal{F}_{att}(\mathcal{BS}(a),SEQ_{\mathcal{SF}}(\mathcal{R}^{-}(a))),\mathcal{F}_{supp}(\mathcal{BS}(a),SEQ_{\mathcal{SF}}(\mathcal{R}^{+}(a)))).$$

To define $\mathcal{F}_{att}$ and $\mathcal{F}_{supp}$, the case of a single attacker (supporter) with a score $v \neq 0$ is firstly expressed in the following equations, based on the intuition that an attacker’s (supporter’s) contribution to the argument’s score decreases (increases) it by an amount proportional to both the attacker’s (supporter’s) score and the argument’s previous score, where $v_0$ is the base score:
$$
f_{att}(v_{0},v) = v_{0} - v_{0} \cdot v = v_{0} \cdot (1-v);\ f_{supp}(v_{0},v) = v_{0} +(1-v_{0}) \cdot v = v_{0} + v -v_{0} \cdot v.
$$
Further,  to deal with those sequences that are empty or consist of all zeros, a notion of ineffective sequence is introduced. A special value $nil \notin \mathbb{I}$ is defined to deal with ineffective sequences returned by $\mathcal{F}_{att}$ and $\mathcal{F}_{supp}$. 
Let $*$ be either `att' or `supp', for a non-ineffective sequence $S \in \mathbb{I}^{*}$, if $S=(v)$, $\mathcal{F}_{*}(v_{0},S) = f_{*}(v_{0},v)$; if $S = (v_{1},...,v_{n})$, $\mathcal{F}_{*}(v_{0},(v_{1},...,v_{n}))= f_{*}(\mathcal{F}_{*}(v_{0},(v_{1},...,v_{n-1})),v_{n})$. $\mathcal{F}_{*}$ produces the same result for any permutation of the same sequence.

The score function $\mathcal{SF}$ can be defined through operator $g: \mathbb{I} \times \mathbb{I} \cup \{nil\} \times \mathbb{I}\cup\{nil\} \rightarrow \mathbb{I}$. Specifically, $g(v_{0},v_{a},v_{s}) = v_{a} \ \text{if} \ v_{s} = nil \ and \ v_{a} \neq nil$; $g(v_{0},v_{a},v_{s}) = v_{s} \ \text{if} \ v_{a} = nil \ and \ v_{s} \neq nil$; $g(v_{0},v_{a},v_{s}) = v_{0} \ \text{if} \ v_{a}=v_{s} =nil$; $g(v_{0},v_{a},v_{s}) = \frac{v_{a}+v_{s}}{2} \ otherwise$.

\section{The QuAM mechanism}
Based on the notion of  QuAD framework,  in this section, we propose an argumentation mechanism to support a mediator in assessing the extent to which the conflict between disputing parties is reduced after presenting arguments, and in discovering the acceptable values for both parties. 

In a mediation dispute, two parties disagree on different values to be assigned to a variable, e.g.,  the payment of a sum from a party to another, the right to pass on a private road, the time of access to a resource, or the child custody in a divorce. The mediation goal is to identify a value acceptable to both parties. To realize this purpose, we first define the input of our mechanism, including a quantitative argumentation mediate (QuAM) framework presented by each party and a set of arguments proposed by the mediator. Next, we define an automatic evaluation method for calculating the degree of acceptability of the goal argument in a QuAM Framework. Finally, we introduce a function for mapping the acceptability degree of the goal argument to the value of a dispute variable, and propose principles for successful mediation.

\subsection{Input of the QuAM mechanism}

Based on the above description, our first input consists of two QuAM frameworks presented by two parties at stage $0$. Note that the mediator can only add one argument each time. We denote $m \in \mathbb{N}$ as the number of moves made by the mediator. The mediation is said to be at stage $m$ after $m$ moves, with the initial stage being stage $0$. Each framework represents arguments (including a goal argument, attackers and supporters) and the original degrees of acceptability assigned by a party to these arguments and relations between the arguments. We assume the goal argument remains fixed throughout the mediation process. For an argument that is influenced (influence target) by a set of other arguments (influencers), we quantify the relations between the influence target and the influencers to represent their relevance. The QuAM framework is defined as follows:

\begin{definition}
 A quantitative argumentation mediate (QuAM) framework is a 6-tuple \\ $\langle \theta, \mathcal{C},\mathcal{P},\mathcal{R}, \mathcal{BS}, \pi \rangle$ such that (for scale $\mathbb{I}=[0, 1]$) :
 \begin{itemize}
     \item $\goal$ is a primary goal argument that a party aims to achieve in the mediation;
     \item $\mathcal{C}$ is a finite set of con-arguments;
     \item $\mathcal{P}$ is a finite set of pro-arguments; \\
    The sets $\mathcal{C}$ and $\mathcal{P}$ are pairwise disjoint, and $\goal \notin \mathcal{C} \cup \mathcal{P}$;
    \item $\mathcal{R} \subseteq (\mathcal{C} \cup \mathcal{P}) \times (\{\goal\} \cup \mathcal{C} \cup \mathcal{P})$ is a set of acyclic binary relations;
    \item $\mathcal{BS}: (\{\goal\}\cup \mathcal{C} \cup \mathcal{P}) \rightarrow \mathbb{I}$ is a total function; $\mathcal{BS}(\textit{a})$ is the base score of argument $\textit{a}$, indicating the original acceptability of $a$ by a party.
    \item  $\pi: \mathcal{R} \rightarrow \mathbb{I}$ is a total function;
    $\pi((a,b))$ is the strength value on relation $(a,b) \in \mathcal{R}$, assigned by a party.
 \end{itemize} 
\end{definition}
Note that $\mathcal{BS}(\goal) =1$; otherwise, the party would change his goal.

We use a compensation dispute mediation case \cite{r18} as an example to illustrate the QuAM framework.

\begin{example}
    The two parties in this case are a customer, Zhang, and a supermarket. 
    While drawing back the curtains at the entrance of the supermarket, Zhang lost his balance and fell, resulting in a rib fracture. 
    The mediator, after separately discussing the specific situations with both parties, generated two QuAM frameworks at stage $0$ shown in Fig. \ref{fig}, where only arguments and their relations are shown.
    \begin{figure}
    \centering
    \includegraphics[height=0.8in, width=3in]{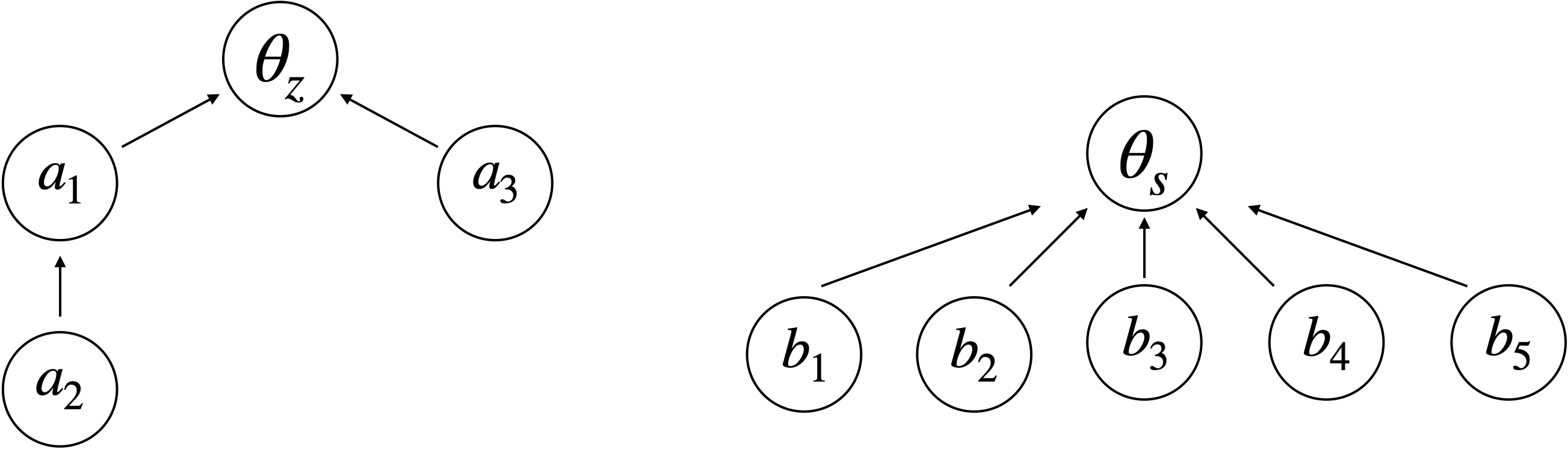}
    \caption{The QuAM framework for Zhang (on the left) and the supermarket (on the right).}
    \label{fig}
    \end{figure} 
    
    First, the arguments, relations, and related strengths proposed by Zhang are introduced. $\goal_{z}$: `The supermarket assumes all responsibility for compensation.', where $\mathcal{BS}(\goal_{z}) = 1$; $\mathcal{P} =\{a_{1}$: `The main reason for the fall was the wet floor.', $a_{2}$: `The supermarket didn't clean the floor in time.', $a_{3}$: `The entrance has no handrails to grab if someone falls.'$\}$, where $\mathcal{BS}(a_{1}) = 0.9$, $\mathcal{BS}(a_{2}) = 0.7$, $\mathcal{BS}(a_{3}) = 0.9$; $\mathcal{C} = \emptyset$; $\mathcal{R} =\{(a_{1},\goal_{z}),(a_{3},\goal_{z}),(a_{2},a_{1})\}$, where $\pi((a_{1},\goal_{z})) = 0.9$, $\pi((a_{3},\goal_{z})) = 0.5$, $\pi((a_{2},a_{1})) = 0.9$.
    
    Second, the arguments, relations, and related strengths proposed by the supermarket. $\goal_{s}$: `The supermarket assumes $20 \%$ responsibility for compensation.' where $\mathcal{BS}(\goal_{s}) = 1$; $\mathcal{P} =\{b_{1}$: `No need for a non-slip mat due to the good weather.', $b_{2}$: `A slippery floor sign at the entrance fulfilled the duty to provide a safety warning.', $b_{3}$: `Zhang is elderly, making it risky to go out alone, so his children have neglected their duty of care.', $b_{4}$: `Zhang did not fall inside the supermarket.', $b_{5}$ = `Zhang was not pushed by anyone in the supermarket.'$\}$, where $\mathcal{BS}(b_{1}) = 0.9$, $\mathcal{BS}(b_{2}) = 0.9$, $\mathcal{BS}(b_{3}) = 0.7$, $\mathcal{BS}(b_{4}) = 0.9$, and $\mathcal{BS}(b_{5}) = 0.7$; $\mathcal{C} = \emptyset$; $\mathcal{R} =\{(b_{1},\goal_{s}),(b_{2},\goal_{s}),(b_{3},\goal_{s}),(b_{4},\goal_{s}), (b_{5},\goal_{s})\}$, where $\pi((b_{1},\goal_{s})) = 0.5$, $\pi((b_{2},\goal_{s})) = 0.7$, $\pi((b_{3},\goal_{s})) = 0.7$, $\pi((b_{4},\goal_{s})) = 0.9$, $\pi((b_{5},\goal_{s})) = 0.4$.
    
    \label{ex1}
\end{example}

The second input of the QuAM mechanism is a set of arguments proposed by the mediator, which can be categorized into two categories. The first type includes arguments based on facts, and can be categorized into opinion arguments and factual arguments. These facts can be presented by the disputants and proven by the mediator based on evidence, or proposed by the mediator directly. Opinion arguments are generated from opinions based on facts, denoted as $A_{O}$, and factual arguments are the facts themselves, denoted as $A_{F}$. The second type includes arguments based on legal norms and facts. We categorize this type of arguments into mandatory arguments and dispositive arguments referring to public order and good morals, which is an important principle in the civil laws of countries around the world since modern times. As we mentioned in the introduction, the law will not intervene proactively in mediation unless public order and good morals are violated. To determine whether a certain legal act violates public order and good morals, the primary consideration is whether it violates the general moral values or the general interests of a society. Based on this standard, criteria can be established to determine actions that violate public order and good morals in civil legal activities, namely the harmfulness of legal acts and the invalidity of contracts. The harmfulness of legal acts mainly refers to behaviors that violate ethical principles, contradict the concept of justice, deprive or severely restrict personal freedom, and harm general social order. The invalidity of contracts includes those such as loan agreements made for gambling, contracts that use a person as collateral and restrict personal freedom, and contracts made after a divorce that severs the relationship between a child and their parents \cite{r19}. Mandatory arguments, denoted as $A_{NM}$, are those based on legal norms that fall within the scope of public order and good morals. In contrast, dispositive arguments, denoted as $A_{ND}$, are based on legal norms that do not fall within this scope, allowing the parties to decide whether and to what extent they accept these arguments. An overall persuasive argument set $A_{P}= A_{O} \cup A_{F} \cup A_{NM} \cup A_{ND}$ and $A_{O} \cap A_{F} \cap A_{NM} \cap A_{ND} = \emptyset$.
A persuasive argument set for party $a$ is denoted as $P_{a} \subseteq A_{P}$. In addition, we assume that the mediator knows the relations between persuasive arguments and arguments in a QuAM framework. New arguments might be raised by two parties in the remaining process of mediation. For simplicity, we do not consider such arguments in this paper.  
A continuation of Example \ref{ex1} is used to illustrate these ideas.

\begin{example}
    $P_{s} = \{p_{1},p_{2},p_{3}, p_{4},p_{5},p_{6}\}$ is the persuasive argument set for the supermarket, where $p_{1} = a_{2}$, $p_{2} = a_{3}$, $p_{3}:$ `The supermarket’s anti-slip sign is blocked.', $p_{4}:$ `Although the stairs are outside the supermarket, they are still part of its business premises.', $p_{5}:$ `The supermarket’s unloading operations made the floor slippery and Zhang’s injury is related to the slippery floor.', $p_{6}:$ `The supermarket’s failure to ensure safety makes it partly liable for compensation.'. $A_{O} = \{p_{4}, p_{5}\}$, $A_{F} = \{p_{1}, p_{2},p_{3}\}$, $A_{NM} = \{p_{6}\} \ and \ A_{ND}= \emptyset$. ${p_{6}}$ is derived from the legal norm, `Operators or managers of business premises and public places, ..., who fail to fulfill their duty of safety protection and thereby cause harm to others, shall assume tort liability'. This norm aims to protect public safety and social order, ensuring that operators and managers are responsible for the safety of the public and falls within the scope of public order and good morals. 

    An example for a dispositive argument is the con-argument, `a child may take the surname of either party', of the argument `according to custom, a child should take his father’s surname'. This con-argument is based on the norm, `a child may take either his father’s surname or his mother’s surname'. The norm grants autonomy to families rather than serving as a mandatory requirement for maintaining public order and good moral. Therefore, it does not fall within the scope of public order and good morals.
\label{ex2}
\end{example}

\subsection{Automatic evaluation in a QuAM framework}
The method for obtaining the final acceptability of an argument is based on the automatic evaluation in a QuAD framework \cite{r16}. The differences between the evaluation in QuAM and QuAD are 
additional considerations of the weight of relations and the characteristics of different types of arguments when quantifying the acceptability of an argument. 
We denote as $SEQ_{\pi}(\mathcal{R}^{-}(a)) = (\pi((a_{1},a)),...,\pi((a_{n},a)))$ and $SEQ_{\pi}(\mathcal{R}^{+}(a)) = (\pi((b_{1},a)),...,\pi((b_{m},a)))$ the corresponding sequence of the weight of relations. The generic score function for an argument $a$ becomes: 
\begin{align}
    &\mathcal{SF}(a) = g(\mathcal{BS}(a), \mathcal{F}_{att}(\mathcal{BS}(a), SEQ_{\pi}(\mathcal{R}^{-}(a)),SEQ_{\mathcal{SF}}(\mathcal{R}^{-}(a))),\\
    \nonumber
    &\quad \quad \quad \quad \mathcal{F}_{supp}(\mathcal{BS}(a),SEQ_{\pi}(\mathcal{R}^{+}(a)),SEQ_{\mathcal{SF}}(\mathcal{R}^{+}(a))))
\end{align}
The degree to which a supporter (resp. attacker) increases (resp. decreases) the acceptability of an argument is proportional to the influence of the supporter (resp. attacker) and inversely proportional (resp. proportional) to the argument’s previous acceptability. 
The influence of an influencer is proportional to its acceptability and the weight of the corresponding relation. In extreme cases, if a party does not accept an argument, even if it is 100\% relevant to the influence target, the target’s acceptability will remain unchanged. Conversely, if the argument is irrelevant to the target, even if the party fully accepts it, the target’s acceptability will not change. To satisfy the above, we can derive the following equations to calculate the acceptability of an argument that has only one supporter or attacker, where $v_{0}$ is the base score  of the argument, $\pi$ is the weight of the support (attack) relation, and $v$ is the acceptability degree of the supporter (attacker):
\begin{align}
   f_{supp}(v_{0},\pi, v) = v_{0} + (1-v_{0}) \cdot (\pi \cdot v); \quad
    f_{att}(v_{0},\pi, v) = v_{0} -v_{0} \cdot (\pi \cdot v) 
\end{align}

If $S=(v)$ and $\Pi=(\pi)$, $\mathcal{F}_{*}(v_{0},S)=f_{*}(v_{0},\pi,v)$; if $S=(v_{1},...,v_{n})$ and $\Pi = (\pi_{1},...,\pi_{n})$, $\mathcal{F}_{*}(v_{0},(\pi_{1},...,\pi_{n}),(v_{1},...,v_{n}))=f_{*}(\mathcal{F}_{*}(v_{0},(\pi_{1},...\pi_{n}),(v_{1},..,v_{n-1})),\pi_{n},v_{n}).$
Now we consider the characteristics of different types of arguments. As a premise, we assume that a rational person will always acknowledge the facts, and according to the core principles of civil law, informal mediation cannot violate public order and good morals. If a party is unwilling to comply with public order and good morals, they will be informed that mediation cannot resolve the dispute, and  be asked to seek assistance from the court, along with being informed of the potential consequence. Therefore, in this work, the base scores of $A_{F}$ and $A_{NM}$ are always equal to 1, and they do not have attackers.  Meanwhile, there is no conflict between mandatory arguments in the context of the same law or factual arguments. In the context of different laws, if conflicting norms lead to conflicting mandatory arguments, then the priority relation over the set of norms can be used to choose a set of conflict-free mandatory arguments \cite{r21}. The constraints of function $g$ in Formula $(1)$ are defined based on $A_{F}$ and $A_{NM}$.

\begin{itemize}
    \item[(C1)] If $b \in \mathcal{R}^{-}(a)$, $b \in A_{F} \cup A_{NM}$ and $\pi((b,a)) = 1$, then $\mathcal{SF}(a)=_{\mathrm{def}}0$;
    \item[(C2)] If $b \in \mathcal{R}^{+}(a)$, $b \in A_{F} \cup A_{NM}$ and $\pi((b,a)) = 1$, then $\mathcal{SF}(a) =_{\mathrm{def}} 1$.
\end{itemize}

\begin{proposition}
Let  $A_{F}$ and $A_{NM}$ be conflict-free sets. 
Given $\mathrm{C1}$ and $\mathrm{C2}$, if there exists $b \in \mathcal{R}^{-}(a)$, such that $b \in A_{F} \cup A_{NM}$ and $\pi((b,a)) = 1$, then there exists no $c \in \mathcal{R}^{+}(a)$, such that $c \in A_{F} \cup A_{NM}$, $\pi((c,a)) = 1$.
\end{proposition}

\begin{proof}
Assume that there exists $c \in \mathcal{R}^{+}(a)$, s.t. $c \in A_{F} \cup A_{NM}$, $\pi((c,a))=1$. Then according to C2, when $\mathcal{SF}(c) =1$, $\mathcal{SF}(a) =1$. According to C1, since $\pi((b,a))=1$ and $\mathcal{SF}(b) =1$, $\mathcal{SF}(a) =0$, contradicting $\mathcal{SF}(a) =1$.
\end{proof}


To specifically demonstrate the aggregation process for obtaining the final acceptability of an argument, we continue with Example \ref{ex2}.
\begin{example}
    The mediator put forward $p_{6}$ to persuade the supermarket. Since $p_{6}$ is a mandatory argument, $\mathcal{BS}(p_{6})= 1$. Assigned by the mediator, $\pi((p_{6},\goal_{s}))=0.5$.\\  
    $\mathcal{SF}(\goal_{s}) = g(1, \mathcal{F}_{att}(1,SEQ_{\pi}((p_{6})),SEQ_\mathcal{SF}((p_{6}))), \mathcal{F}_{supp}(1,SEQ_{\pi} ((b_{1},b_{2},b_{3},b_{4},b_{5})),\\SEQ_\mathcal{SF} ((b_{1},b_{2},b_{3},b_{4},b_{5}))))$;
    $\mathcal{SF}(b_{1}) = 0.9$; 
    $\mathcal{SF}(b_{2})=0.9$; 
    $\mathcal{SF}(b_{3})=0.7$;
    $\mathcal{SF}(b_{4})=0.9$; 
    $\mathcal{SF}(b_{5})= 0.7$; 
    $\mathcal{SF}(p_{6})= 1$;
    $\pi(b_{1},\goal_{s})= 0.5$;
    $\pi(b_{2},\goal_{s})= 0.7$;
    $\pi(b_{3},\goal_{s})= 0.7$;
    $\pi(b_{4},\goal_{s})= 0.9$;
    $\pi(b_{5},\goal_{s})= 0.4$.
    $\mathcal{F}_{supp} = 1$, $\mathcal{F}_{att} = 0.5$, and $\mathcal{SF}(\goal_{s})= 0.75$.
    \label{ex3}
\end{example}

\subsection{Transformation function}
In the previous section, we have introduced a method for calculating the acceptability of a goal argument in mediation. However, this acceptability cannot be directly used to assess the specific value of the variable associated with the goal argument. For instance, while the acceptability of the goal `full compensation' might decrease after being attacked by new arguments raised by the mediator, it doesn’t reveal the exact amount the party would accept. A common approach is to propose a new argument, for instance, `half the compensation', and redo the mediation process. This, however, can make the process cumbersome, as the mediator must rely on his subjective judgement to continually adjust goals until one is accepted. To address this, we aim to explore the relationship between a goal argument's acceptability and the variable associated with it in different mediation scenarios. The variable that two parties struggle with is called a decision variable. Each decision variable is associated with a set of possible values, and the value that a party aims to achieve is called a target value.
When categorizing decision variables, we follow the civil litigation classification, which divides cases into actions for confirmation, for performance, and for modification \cite{r20}. This classification is adopted because non-litigation dispute resolution relies on legitimate civil claims, which are typically reflected in civil litigation categories.
Accordingly, we classify the decision variables into binary and continuous variables, and provide transformation functions to map argument acceptability degrees to possible values, thereby simplifying the mediation process.


Binary variables are discrete, 
aligning with actions for confirmation and for modification. The two subcategories of binary variables are unilateral and joint variables. Binary unilateral variables ($BUV$) involve decisions made by a single party, e.g., whether to remove a fence to accommodate a neighbor. Binary joint variables ($BJV$) involve an indivisible property or right, requiring a decision on which party will take full holdership of the property or right, e.g., who has custody of the child. The transformation functions map an acceptability degree  to a possible value $v \in V$, where $V$ is a set of all possible values of a decision variable. For binary variables, $V = \{\mathbf{0}, \mathbf{1}\}$. For unilateral variables, $\mathbf{0}$ stands for `No' and $\mathbf{1}$ stands for `Yes'. For joint variables, $\mathbf{0}$ stands for ‘the other party’ and $\mathbf{1}$ stands for ‘the party itself’. Since there are only two contradictory values, a decrease in the acceptability of one value will lead to an increase in the acceptability of the other, and vice versa. We use $\goal^{\mathbf{0}}$ and $\goal^{\mathbf{1}}$ to represent the goal arguments corresponding to  values $\mathbf{0}$ and $\mathbf{1}$. For instance, concerning whether to remove a fence, for the party whose attitude is `Yes', the goal  argument of this party is denoted as $\goal^{\mathbf{1}}$. Note that both parties may have the same attitude. In this case, there is no dispute. 
\begin{definition}
    Let $k$ be a threshold when a party accepts neither $\mathbf{0}$ nor $\mathbf{1}$. The transformation functions of $BUV$ (the same for $BJV$) for $\goal^\mathbf{1}$ and $\goal^\mathbf{0}$,  denoted as $\tau_1$ and $\tau_0$ respectively, mapping each acceptability value to a value in $V$, are defined  as follows:
    \[
    \tau_{1}(\mathcal{SF}(\goal^\mathbf{1})) = 
    \left\{  
         \begin{array}{lr}  
           0 \ \ \ \ \ \ \ \ \    \mathcal{SF}(\goal^\mathbf{1}) \in [0, k)  \\   
           1 \ \ \ \ \ \ \ \ \    \mathcal{SF}(\goal^\mathbf{1}) \in (k,1]   
         \end{array}  
    \right.;  
    \quad 
    \tau_{0}(\mathcal{SF}(\goal^\mathbf{0})) = 
    \left\{  
         \begin{array}{lr}  
         1 \ \ \ \ \   \ \ \mathcal{SF}(\goal^\mathbf{0}) \in [0, k)  \\   
         0 \ \ \ \ \ \ \   \mathcal{SF}(\goal^\mathbf{0}) \in (k,1]     
         \end{array}  
    \right.  
    \]
\end{definition}

Obviously, we have the following proposition. 

\begin{proposition}
Both $\tau_1$ and $\tau_0$ are monotonic. 
\end{proposition}

Continuous variables are related to liability or damages, aligning with actions for performance, e.g., compensation amounts can be adjusted within a certain range. These variables can also be classified as either unilateral or joint. In this case, $V = [0,1]$, where $0$ stands for `0 \%' and $1$ stands for `100 \%'. Continuous unilateral variables ($CUV$) refer to one party providing divisible property or resources to the other, e.g., how much compensation. 

\begin{definition}
    Let $x, y\in V$ be the target values of a payer and a payee respectively. Let $k_{payer}$ be the degree of the acceptability of a payer when he needs to give 100\%, and $k_{payee}$ be the degree of the acceptability of a payee when he gets 0\%. The transformation functions of $CUV$ for the payer and the payee respectively are defined as follows:
    \[
    \mu_{payer}: [k_{payer},1] \rightarrow [x,1];
    \quad
    \mu_{payee}: [k_{payee},1] \rightarrow [0,y]
    \]
\end{definition}

Continuous joint variables ($CJV$) involve both parties  sharing the use or benefits of the same property or resource, i.e., shared profits from a resource. 
\begin{definition}
    Let $x,y\in V$ be the target values of party 1 and party 2 respectively. Let $k_{p1}$ and $k_{p2}$ be the degrees of the acceptability of the two parties when they get 0\%. The transformation functions of $CJV$ for the two parties respectively are defined as follows:
    \[
    \mu_{p1}: [k_{p1},1] \rightarrow [0,x];
    \quad
    \mu_{p2}: [k_{p2},1] \rightarrow [0,y]
    \]
\end{definition}
Since functions for $CUV$ and $CJV$ are domain-dependent, we propose the following guiding principles they should follow.
\begin{principle}
Let $\alpha$ and $\beta$ be the acceptability degrees of a goal argument at different stages. If $\alpha \leq \beta$, then $\mu_{payer}(\alpha) \geq \mu_{payer}(\beta)$, $\mu_{payee}(\alpha) \leq \mu_{payee}(\beta)$, $\mu_{p1}(\alpha) \leq \mu_{p1}(\beta)$ and $\mu_{p2}(\alpha) \leq \mu_{p2}(\beta)$.

\end{principle}



\begin{principle}
    If $x$ and $y$ are the target values of a payer and a payee respectively, then $x < y$; if $x$ and $y$ are the target values of party 1 and party 2 respectively, then $x+y>1$.
\end{principle}

Now we put forward conditions for reaching a consensus in the following four situations. For $BUV$, both parties must achieve the same outcome; for $BJV$, they must achieve opposite outcomes; for $CUV$, the amount the payer wants to pay must be greater than or equal to the amount requested by the payee; and for $CJV$, the total amount desired by both parties cannot exceed the available property or resources. 
\begin{definition}
    Let $\theta_{1}$ and $\theta_{2}$ be the goal arguments for the two parties respectively. In legal mediation, two parties reach consensus if and only if: 
    \begin{itemize}
        \item $\tau_{1}(\mathcal{SF}(\theta_{1})) = \tau_{0}(\mathcal{SF}(\theta_{2}))$ for $BUV$
        \item $\tau_{1}(\mathcal{SF}(\theta_{1})) = 1 - \tau_{0}(\mathcal{SF}(\theta_{2}))$ for $BJV$
        \item $\mu_{payer}(\mathcal{SF}(\theta_{1})) \geq \mu_{payee}(\mathcal{SF}(\theta_{2}))$ for $CUV$
        \item $\mu_{p1}(\mathcal{SF}(\theta_{1})) + \mu_{p2}(\mathcal{SF}(\theta_{2})) \leq 1$ for $CJV$
    \end{itemize}

\end{definition}

\begin{example}
    Followed by Example \ref{ex3}, a possible transformation function can be $\mu_{payer}(\alpha) = -0.8 \alpha + 1$ with $\alpha \in [0,1]$ and $\mu_{payer}(\alpha) \in [0.2,1]$. $\mu_{payer}(\mathcal{SF}(\goal_{s}))=\mu_{payer}(0.75)=0.4 \leq \mu_{payee}(\mathcal{SF}(\goal_{z}))=1$, the mediator still needs to continue mediation.
\end{example}

\section{Properties}
Since the QuAM framework is defined under the QuAD framework, all the properties proposed in \cite{r16} are satisfied. In this section, we propose some properties unique to the QuAM mechanism. Before stating the properties, we introduce functions measuring the conflict between two parties.
\begin{definition}
    The distance between goals of two parties $\goal_{1}$ and $\goal_{2}$. 
    \begin{itemize}
        \item $BUV$: $\mathcal Dist$$(\goal_{1},\goal_{2})= \left| \tau_{1}(\mathcal{SF}(\goal_{1})) - \tau_{0}(\mathcal{SF}(\goal_{2})) \right|$.
        \item $BJV$: $\mathcal Dist$$(\goal_{1},\goal_{2})= 1- \left| \tau_{1}(\mathcal{SF}(\goal_{2})) - \tau_{0}(\mathcal{SF}(\goal_{1})) \right|$.
        \item $CUV$: $\mathcal Dist$$(\goal_{1},\goal_{2})= \mu_{payee}(\mathcal{SF}(\goal_{2}))-\mu_{payer}(\mathcal{SF}(\goal_{1}))$ for $\mu_{payee} \geq \mu_{payer}$; \\ otherwise, $\mathcal Dist$$=0$.
        \item $CJV$: $\mathcal Dist$$(\goal_{1},\goal_{2})= (\mu_{p1}(\mathcal{SF}(\goal_{1}))+\mu_{p2}(\mathcal{SF}(\goal_{2})))-1$ for $(\mu_{p1}(\mathcal{SF}(\goal_{1}))+\mu_{p2}(\mathcal{SF}(\goal_{2}))) \geq 1$; otherwise, $\mathcal Dist$$=0$
    \end{itemize}
   \label{def4}
\end{definition}

\begin{proposition}
Let$\langle\theta,\emptyset,\mathcal{P},\mathcal{R},\mathcal{BS},\pi\rangle$ be a QuAM framework for a party, $a$ be an argument put forward by a mediator to persuade the party, and $\theta'$ be a goal argument of another party. If $\mathcal Dist$$(\goal,\theta^\prime)$ decreases, then $a \in \mathcal{R}^{-}(\goal)$.
\end{proposition}

\begin{proof}
    If $a \notin \mathcal{R}^{-}(\theta)$, then $a\in \mathcal{R}^{+}(p)$ where $p \in \{\goal\} \cup \mathcal{P}$ or $a \in \mathcal{R}^{-}(c)$ where $c \in \mathcal{P}$. If $a \in \mathcal{R}^{+}(p)$, since $\mathcal{BS}(\theta)=1$, $\mathcal{SF}(\goal) = \mathcal{F}_{supp}(1,\Pi,S) = 1 = \mathcal{BS}(\theta)$. $\mathcal{SF}(\goal')$ remains unchanged as it has no additional influencers, and $\mathcal Dist$$(\goal,\theta')$ remains unchanged. If $a \in \mathcal{R}^{-}(c)$, $\mathcal{SF}(c) \in [0,1]$, i.e. attacking $c$ can at most cause it to have no influence on a goal argument but cannot change $c$ from a supporter to an attacker. Therefore, $\mathcal{SF}(\goal) = \mathcal{F}_{supp}(1,\Pi,S) = 1$ remains unchanged and $\mathcal Dist$$(\goal,\theta')$ remains unchanged.
\end{proof}
\begin{proposition}
    If $\theta$ is a goal argument, and the value associated with $\theta$ at stage 0 is still its target value. Let $a$ be an argument proposed by a mediator to persuade the party, where $a \in A_{F} \cup A_{NM}$ with $a \in R^{-}(\goal)$, and $\pi((a,\theta))=1$, then without persuading the other party whose goal argument is $\goal'$, $\mathcal Dist$$(\goal,\theta')=0$.
\end{proposition}
\begin{proof}
     Let $\mathcal{SF}_{m}(\theta)$ and $\mathcal{SF}_{m}(\theta')$ be the acceptability of goal arguments $\theta$ and $\theta'$ at stage $m$ respectively. According to constraint C1 of $g$, at stage 1, $\mathcal{SF}_{1}(\goal)=0$. Since $\theta'$ is not being influenced at stage 1, $\mathcal{SF}_{1}(\theta')=\mathcal{SF}_{0}(\theta')$.
     For $BUV$, if $\theta = \theta^{\mathbf{1}}$ and $\theta' = \theta^{\mathbf{0}}$ (similar for the opposite), according to Definition 3.5, at stage 0, $\tau_{1}(\mathcal{SF}_{0}(\theta)) = 1- \tau_{0}(\mathcal{SF}_{0}(\theta')) \neq \tau_{0}(\mathcal{SF}_{0}(\theta'))$. According to Definition 3.2, since the value associated with $\theta$ remains unchanged at stage 0,
     $\mathcal{SF}_{0}(\theta) \in (k,1]$. Since $\mathcal{SF}_{1}(\goal)=0$, $\tau_{1}(\mathcal{SF}_{1}(\theta)) = 1- \tau_{1}(\mathcal{SF}_{0}(\theta))=\tau_{0}(\mathcal{SF}_{1}(\theta'))$. Therefore, $\mathcal Dist$$(\goal,\theta')= 0$ according to Definition 4.1. Similar for $BJV$, $\tau_{1}(\mathcal{SF}_{0}(\theta)) = \tau_{0}(\mathcal{SF}_{0}(\theta'))$ and $\tau_{1}(\mathcal{SF}_{1}(\theta)) = 1- \tau_{1}(\mathcal{SF}_{0}(\theta))\neq \tau_{0}(\mathcal{SF}_{1}(\theta'))$. Therefore, $\mathcal Dist$$(\goal,\theta')= 0$. For $CUV$, when $\theta$ is the goal of a payer, $\mu_{payer}(\mathcal{SF}_{0}(\theta)) < \mu_{payee}(\mathcal{SF}_{0}(\theta'))$. According to Definition 3.3, $\mu_{payer}(\mathcal{SF}_{1}(\theta)) = \mu_{payer}(0)= 1$. Therefore $\mu_{payer}(\mathcal{SF}_{1}(\theta)) \geq \mu_{payee}(\mathcal{SF}_{1}(\theta'))$, and $\mathcal Dist$$(\goal,\theta') = 0$. Similar proof for $\theta$ being the goal of a payee. For $CJV$, $\mu_{p1}(\mathcal{SF}_{0}(\theta)) + \mu_{p2}(\mathcal{SF}_{0}(\theta')) > 1$. According to Definition 3.4, $\mu_{p1}(\mathcal{SF}_{1}(\theta))=\mu_{p1}(0)=0$. Therefore $\mu_{p1}(\mathcal{SF}_{1}(\theta)) + \mu_{p2}(\mathcal{SF}_{1}(\theta')) \leq 1$, and $\mathcal Dist$$(\goal,\theta') = 0$.
\end{proof}



\section{Conclusions and future work}
In this paper, we have proposed a QuAM mechanism to support a mediator in assessing the extent to which the conflict between the disputing parties is reduced after presenting arguments, and in discovering the acceptable values of the goal arguments for both parties. The main contributions of this paper are two-fold. First, we introduce a QuAM framework for determining the acceptability of a mediation goal. Second, we develop a formalism to model the relationship between the acceptability of a goal argument and the values of a decision variable.

There are already some works on computational mediation. Two related works are \cite{r12,r13}. Sierra et al. \cite{r12} introduces a negotiation mediator that combines argumentation and case-based reasoning. The mediator integrates the argumentation systems of both parties through a merging argumentation framework, and generates and justifies solutions to help the negotiating parties, who have failed to reach an agreement, to achieve consensus. Tresca et. al. \cite{r13} presents an argumentation-based mediation model that uses the mediator's knowledge and resources to facilitate the resolution of disputes between negotiating parties. These two works are extensions of negotiation, with the main goal of proposing solutions acceptable for both parties. They didn't consider the process of reaching consensus, making it difficult to assess changes in the parties' acceptance of the outcomes during mediation, and to measure the changes in the degree of conflict in dynamic process of mediation. Moreover, they did not consider the different roles that various arguments play in mediation, and the examples they used are simple cases that fall outside the scope of legal mediation, and therefore cannot be applied to legal mediation scenarios. 

As a new methodology for dispute resolution in legal mediation based on quantitative argumentation, there are some promising research topics worth to be further studied. First, we have not yet explored how mediators should present arguments to resolve disputes most efficiently. So, it is interesting to develop strategies for the mediation process to better assist a mediator in the process of mediation, based on the theory of strategic argumentation \cite{r22}. Second, specific functions of transformation and the thresholds  have not been specified and instantiated yet. Possible future work may include the specification of detailed transformation functions, and effective machine learning methods for learning thresholds. Third, it can be challenging to find suitable datasets. A possible way is to generate synthetic datasets.

\end{document}